\documentclass[a4paper]{article}

\usepackage{graphicx}
\usepackage{subfigure} 
\usepackage{fancyhdr}
\usepackage{fancybox}
\usepackage{natbib}
\usepackage{hyperref}
\usepackage{xcolor}

\usepackage{amsmath}
\usepackage{amssymb}
\usepackage{amsthm}
\usepackage{bbm}

\newcommand{\xbf}{\ensuremath{\mathbf{x}}}

\newcommand{\Cbf}{\ensuremath{\mathbf{C}}}
\newcommand{\Dbf}{\ensuremath{\mathbf{D}}}

\newcommand{\Pfrak}{\ensuremath{\mathfrak{P}}}
\newcommand{\Qfrak}{\ensuremath{\mathfrak{Q}}}

\newcommand{\Fcal}{\ensuremath{\mathcal{F}}}
\newcommand{\Scal}{\ensuremath{\mathcal{S}}}

\newcommand{\posterior}{\Qfrak}

\newcommand{\expectation}{\E}
\newcommand{\indicator}[1]{\I\left(#1\right)}

\newcommand{\sign}{\operatorname{sign}}
\newcommand{\argmax}{\operatorname{argmax}}

\newcommand{\E}{\mathbb{E}}
\newcommand{\R}{\mathbb{R}}
\newcommand{\I}{\mathbb{I}}
\newcommand{\PR}{\operatorname{\mathbb{P}}}
\newcommand{\x}{\times}

\newcommand{\leqo}{\preccurlyeq}

\newcommand{\egaldef}{\!\stackrel{def}{=}\!}

\newtheorem{cor}{Corollary}

\newtheorem{lemma}{Lemma}
\newtheorem{theorem}{Theorem}

\newtheorem{propo}{Proposition}

\begin{document}

\lhead{E. Morvant, S. Ko\c co, L. Ralaivola} 
\rhead{PAC-Bayesian Generalization Bounds on the Confusion Matrix}
\rfoot[Technical Report V 4.0]{\thepage} 
\cfoot{} 
\lfoot[\thepage]{Technical Report V 4.0}

\renewcommand{\headrulewidth}{0.4pt}  
\renewcommand{\footrulewidth}{0.4pt}

\title{PAC-Bayesian Generalization Bound on Confusion Matrix for Multi-Class Classification\thanks{This work was supported in part by the french projects VideoSense ANR-09-CORD-026 and DECODA ANR-09-CORD-005-01 of the ANR in part by the IST Programme of the European Community, under the PASCAL2 Network of Excellence, IST-2007-216886. This publication only reflects the authors' views.}}

 \author{
Emilie Morvant
\and  Sokol Ko\c co 
\and Liva Ralaivola 
\and  Aix-Marseille Univ., LIF-QARMA, CNRS, UMR 7279, F-13013, Marseille, France\\ \textit{\{firstname.name\}@lif.univ-mrs.fr}}

\maketitle

\begin{abstract}
In this work, we propose a PAC-Bayes bound for the generalization risk of the Gibbs classifier in the multi-class classification framework.
The novelty of our work is the critical use of the {\em confusion matrix} of a classifier 
as an error measure; this puts our contribution in the line of work aiming at dealing with performance measure that are richer than mere scalar criterion such as the misclassification rate. Thanks to very recent and beautiful results  on matrix concentration inequalities,  we derive two bounds showing that the true confusion risk of the Gibbs classifier is upper-bounded by its empirical risk plus a term depending on the number of training examples in each class.
To the best of our knowledge, this is the first PAC-Bayes bounds based on confusion matrices.
\end{abstract}


\textbf{Keywords:} Machine Learning, PAC-Bayes generalization bounds, Confusion Matrix, Concentration Inequality, Multi-Class Classification

\section{Introduction}
\label{sec:intro}

The PAC-Bayesian framework, first introduced by \cite{Mcallester99a}, provides an important field of research in learning theory.
It borrows ideas from the philosophy of Bayesian inference and mix them with techniques used in statistical approaches of learning. Given a family of classifiers $\mathcal{F}$, the ingredients of a PAC-Bayesian bound are a {\em prior distribution} $\mathfrak{P}$ over $\mathcal{F}$, a learning sample $S$ and a {\em posterior distribution} $\mathfrak{Q}$ over $\mathcal{F}$. Distribution $\mathfrak{P}$ conveys some prior belief on what are the best classifiers from $\mathcal{F}$ (prior any access to $S$); the classifiers expected to be the most performant for the classification task at hand therefore have the largest weights under $\mathfrak{P}$. The posterior distribution $\mathfrak{Q}$ is learned/adjusted using the information provided by the training set $S$. The essence of PAC-Bayesian   results is to bound the risk of the {\em stochastic} Gibbs classifier associated with $\mathfrak{Q}$ \citep{Catoni2004} ---in order to predict the label of an example ${\bf x}$, the Gibbs classifier first draws a classifier $f$ from $\mathcal{F}$ according to $\mathfrak{Q}$ and then returns $f({\bf x})$ as the predicted label.

When specialized to appropriate function spaces $\mathcal{F}$ and relevant families of prior and posterior distributions,
PAC-Bayes bounds can be used to characterize the error of a few existing classification methods.
An example deals with the risk of methods based upon the idea of the majority vote in the case of binary classification. 
We may notice that if $\mathfrak{Q}$ is the posterior distribution, the error of the $\mathfrak{Q}$-weighted majority vote
classifier, which makes a prediction for ${\bf x}$ according to
$\sum_ff({\bf x})\mathfrak{Q}(f)$, is bounded by twice the error of
the Gibbs classifier.  
If the classifiers from $\mathcal{F}$ on which the distribution
$\mathfrak{Q}$ puts a lot of weight are good enough, then the bound on
the risk of the Gibbs classifier can be an informative bound
for the risk of the $\mathfrak{Q}$-weighted majority vote.  With a more elaborated argument, \citet{Langford02}
give a PAC-Bayes bound for Support Vector Machine (SVM) which closely relates the risk of the Gibbs classifier
and that of the corresponding majority vote classifier, and where the margin of the examples enter into play.  In their study,
both the prior and posterior distribution are normal
distributions, with different means and variances.  Empirical results
show that this bound is a good estimator of the
risk of SVMs \citep{Langford2005}.

PAC-Bayes bounds can also be used to derive new supervised learning algorithms.
For example, \citet{Lacasse07} have introduced an elegant bound on the risk of the majority vote, which holds for any space $\mathcal{F}$.
This bound is used to derive an algorithm, namely {\tt MinCq} \citep{MinCq}, which achieves
empirical results on par with state-of-the-art methods.
Some other important results are given in \citep{Catoni2007,Seeger02,Mcallester99b,Langford2001}.

In this paper, we address the multiclass classification problem.
Some related works are therefore multiclass formulations for the SVMs, such as the frameworks of 
\cite{Weston98}, \cite{Lee04} and \cite{Crammer2002}.
As majority vote methods, we can also cite multiclass adaptations of  {\tt AdaBoost} \cite{Freund1996}, such as the framework proposed by \cite{Mukherjee2010}, {\tt AdaBoost.MH/AdaBoost.MR} algorithms of \cite{Schapire99} and {\tt SAMME} algorithm by \cite{Zhu09}.



The originality of our work is that we consider the {\em confusion matrix} of the Gibbs classifier as an error measure.
We believe that in the multiclass framework, it is more relevant to consider the confusion matrix as the error measure than the mere misclassification error, which corresponds to the probability  for some classifier $h$ to err on ${\bf x}$.
The information as to what is the probability for an instance of class $p$ to be classified into class $q$ (with $p\neq q)$ by some predictor is indeed crucial in some applications (think of the difference between false-negative and false-positive predictions in a diagnosis automated system).
To the best of our knowledge, we are the first to propose a generalization bound on the confusion matrix in the PAC-Bayesian framework.
The result that we propose heavily relies on the matrix concentration inequality for sums of random matrices introduced by \citet{Tropp2011}. 
One may anticipate that generalization bounds for the confusion matrix may also be obtained in other framework than the PAC-Bayesian one, such as the uniform stability framework, the online learning framework and so on.



The rest of this paper is organized as follows. Sec. \ref{sec:notations} introduces the setting of multiclass learning and some of the basic notation used throughout the paper.
Sec. \ref{sec:binarybound} briefly recalls the folk PAC-Bayes bound as introduced in \cite{Mcallester03}. 
In Sec. \ref{sec:bound}, we present the main contribution of this paper, our PAC-Bayes bound on the confusion matrix, followed by its proof in Sec. \ref{sec:preuve}. We discuss some future works in Sec. \ref{sec:discussion}.


\section{Setting and Notations}
\label{sec:notations}
 
This section presents the general setting that we consider and the
different tools that we will make use of. 

\subsection{General Problem Setting}
We consider classification tasks over the {\it input space} $X\!\subseteq\!\R^d$ of dimension $d$.
The {\it output space} is denoted by $Y\!=\!\{1,\dots,Q\}$, where $Q$ is the number of classes.
The learning sample is denoted by $S=\{(\xbf_i,y_i)\}_{i=1}^{m}$ where each example is drawn {\it i.i.d.} from a fixed ---but unknown--- probability distribution $\mathfrak{D}$ defined over $X\x Y$.
$\mathfrak{D}_m$ denotes the distribution of a $m$-sample.
$\Fcal\subseteq \R^X$  is a family of classifiers $f:X\to Y$. 
$\mathfrak{P}$ and $\mathfrak{Q}$ are respectively the {\it prior} and the {\it posterior} distributions over $\Fcal$.
Given the prior distribution $\mathfrak{P}$ and the training set $S$, the learning process consists in finding the posterior distribution $\mathfrak{Q}$ leading to a good generalization.

Since we make use of the prior distribution $\mathfrak{P}$ on $\Fcal$, a PAC-Bayes generalization bound depends on the Kullback-Leibler divergence (KL-divergence):
\begin{align}
\label{eq:KL}
KL(\mathfrak{Q}\|\mathfrak{P}) = \E_{f\sim\mathfrak{Q}}\log{\frac{\mathfrak{Q}(f)}{\mathfrak{P}(f)}}.
\end{align}

The function $\sign(x)$ is equal to $+1$ if $x\geq 0$ and $-1$ otherwise.
The indicator function $\I(x)$ is equal to $1$ if $x$ is true and $0$ otherwise.

\subsection{Conventions and Basics on Matrices}
Throughout the paper we consider only real-valued square  matrices
$\mathbf{C}$ of order $Q$ (the number of classes).
${}^t\mathbf{C}$ is the transpose of the matrix $\mathbf{C}$, 
$\mathbf{Id}_Q$ denotes the identity matrix of size $Q$ and
$\mathbf{0}$ is the zero matrix.

The results given in this paper are based on a concentration inequality of \citet{Tropp2011} for a sum of random self-adjoint matrices.
In the case when a matrix is not self-adjoint and is real-valued, we use the dilation of such a matrix, given in \citet{paulsen2002}, which is defined as follows:



\begin{equation}
\label{eq:dilatation}
\mathcal{S}(\mathbf{C}) \egaldef 
\left(
\begin{array}{cc}
\mathbf{0} & \mathbf{C}\\
{}^t\mathbf{C} & \mathbf{0}
\end{array}
\right)\!.
\end{equation}


The symbol $\|\cdot\|$ corresponds to the {\it operator norm} also called the {\it spectral norm}: it returns the largest singular value of its argument, which is defined by
\begin{align}
\label{eq:norm}
\|\mathbf{C}\|=\max\{\lambda_{\max}(\mathbf{C}),-\lambda_{\min}(\mathbf{C})\},
\end{align}
 where $\lambda_{\max}$ and $\lambda_{\min}$ are respectively the algebraic maximum and minimum singular value of $\mathbf{C}$.
Note that the dilation preserves spectral information, so we have:
\begin{align}
\label{eq:preserve}
\lambda_{\max}\big(\mathcal{S}(\mathbf{C})\big)=\|\mathcal{S}(\mathbf{C})\|=\|\mathbf{C}\|.
\end{align}
Since $\|\cdot\|$ is a regular norm, the following equality obviously holds:
\begin{align}
\label{eq:scalar}
\forall a \in \R,\ \|a\Cbf\| = |a|\|\Cbf\|.
\end{align}
Given the matrices $\Cbf$ and $\Dbf$ both made of nonnegative elements and
such that $0\leq \Cbf\leq \Dbf$ (element-wise), we have:
\begin{align}
\label{eq:sp}
0\leq \Cbf \leq \Dbf \Rightarrow \|\Cbf\|\leq \|\Dbf\|.
\end{align}


\section{The Usual PAC-Bayes Theorem}
\label{sec:binarybound}

In this section, we recall the main PAC-Bayesian bound in the binary classification case as presented in \citep{Mcallester03,Seeger02,Langford2005}.
The set of labels we consider is $Y=\{-1;1\}$ (with $Q=2$) 
and, for each classifier $f\in\mathcal{F}$, the predicted output of $\xbf\in X$ is given by $\sign(f(\xbf))$. 
The true risk $R(f)$ and the empirical error $R_S(f)$ of $f$ are defined as:
\begin{align*}
R(f)\egaldef \E_{(\xbf,y)\sim \mathfrak{D}} \I(f(\xbf\ne y))\qquad ; \qquad  R_S(f)\egaldef \frac{1}{m} \sum_{i=1}^m \I(f(\xbf_i\ne y_i)).
\end{align*}
The learner's aim is to choose a posterior distribution $\mathfrak{Q}$ on $\mathcal{F}$ such that the risk of the $\mathfrak{Q}$-weighted majority vote (also called the Bayes classifier) $B_{\mathfrak{Q}}$ is as small as possible.
$B_{\mathfrak{Q}}$ is defined by:
\begin{align*}
B_{\mathfrak{Q}}(\xbf) = \sign\left[\E_{f\sim\mathfrak{Q}} f(\xbf)\right].
\end{align*}
The true risk $R(B_{\mathfrak{Q}})$ and the empirical error $R_S(B_{\mathfrak{Q}})$ of the Bayes classifier are defined as the probability that it commits an error on an example:
\begin{align}
\label{eq:bayesclassifier}
R(B_{\mathfrak{Q}}) \egaldef \PR_{(\xbf,y)\sim \mathfrak{D}} \left(B_{\mathfrak{Q}}(\xbf)\ne y\right).
\end{align}
However, the PAC-Bayes approach does not directly bound the risk of $B_{\mathfrak{Q}}$.
Instead, it bounds the risk of the stochastic Gibbs classifier
$G_{\mathfrak{Q}}$ which predicts the label of $\xbf\in X$ by first
drawing $f$ according to $\mathfrak{Q}$ and then returning $f(\xbf)$.
The true risk $R(G_{\mathfrak{Q}}) $ and the empirical error
$R_S(G_{\mathfrak{Q}}) $ of $G_{\mathfrak{Q}}$ are therefore:
\begin{align}
\label{eq:gibbsclassifier}
R(G_{\mathfrak{Q}})=\E_{f\sim\mathfrak{Q}} R(f)\qquad ;\qquad  R_S(G_{\mathfrak{Q}})=\E_{f\sim\mathfrak{Q}} R_S(f).
\end{align}
Note that in this setting, if $B_{\mathfrak{Q}}$ misclassifies $\xbf$, then at least half of the classifiers (under $\Qfrak$) commit an error on $\xbf$.
Hence, we directly have: $R(B_{\mathfrak{Q}})\leq 2R(G_{\mathfrak{Q}})$.
Thus, an upper bound on $R(G_{\mathfrak{Q}})$ gives rise to an upper bound on $R(B_{\mathfrak{Q}})$.

We present the PAC-Bayes theorem which gives a bound on the error of the stochastic Gibbs classifier.
\begin{theorem}[{\it i.i.d.} binary classification PAC-Bayes Bound]
\label{theo:binaryborne}
For any $\mathfrak{D}$, any $\mathcal{F}$, any $\mathfrak{P}$ of support $\mathcal{F}$, any $\delta\in (0,1]$, we have,
\begin{align*}
\PR_{S\sim \mathfrak{D}_m}  \Bigg(\forall \mathfrak{Q}\textrm{ on }\mathcal{F},\ kl\big(R_S(G_{\Qfrak}),R(G_{\Qfrak})\big) \leq 
 \frac{1}{m}\bigg[KL(\mathfrak{Q}\|\mathfrak{P})+\ln\frac{\xi(m)}{\delta}\bigg]\Bigg)& \geq  1-\delta,
\end{align*}
where $kl(a,b)\egaldef a\ln\frac{a}{b}+(1-a)\ln\frac{1-a}{1-b}$, and $\xi\egaldef \sum_{i=0}^m \binom{m}{i}(i/m)^i(1-i/m)^{m-i}$.
\end{theorem}

We now provide a novel PAC-Bayes bound in the context of multiclass classification by considering the confusion matrix as an error measure.


\section{Multiclass PAC-Bayes Bound}
\label{sec:bound}

\subsection{Definitions and Setting}
As said earlier, we focus on multiclass classification. The output space is $Y = \{1,\dots,Q\}$, with $Q>2$. 
We only consider learning algorithms acting on learning sample $S=\{(\xbf_i,y_i)\}_{i=1}^{m}$ where each example is drawn {\it i.i.d} according to $\mathfrak{D}$, such that $|S|\geq Q$ and $m_{y_j}\geq 1$ for every class $y_j \in Y$, where $m_{y_j}$ is the number of examples of real class $y_j$.
In the context of multiclass classification, an error measure can be the {\em confusion matrix}.
Especially, we consider a confusion matrix builds upon the classical definition  based on conditional probalities:
It is inherent (and desirable)  to 'hide' the effects of diversely represented classes.
Concretely, for a given classifier $f\in\mathcal{F}$ and a sample $S = \{ (\xbf_i,y_i) \}_{i=1}^m \sim \mathfrak{D}_m$, the {\em empirical confusion matrix} $\mathbf{D}_S^f=(\hat{d}_{pq})_{1\leq p, q \leq Q}$ of $f$ is defined as follows:
\begin{align*}
\forall (p,q),\ \hat{d}_{pq} \egaldef 
\displaystyle\sum_{i=1}^m\frac{1}{m_{y_i}}\I(f(\xbf_i)=q)\I(y_i=p).
\end{align*}
The {\em true confusion matrix} $\mathbf{D}^f=(d_{pq})_{1\leq p,q \leq Q}$ of $f$ over $\mathfrak{D}$ corresponds to:
\begin{align*}
\forall (p,q),\ d_{pq} \egaldef\ &\E_{\xbf|y=p}\I\big(f(\xbf)=q\big)\\
=\ &\PR_{(\xbf,y)\sim \mathfrak{D}}( f(\xbf)=q |y=p ).
\end{align*}

If $f$ correctly classifies every example of the sample $S$, then all the elements of the confusion matrix are $0$, except for the diagonal ones which correspond to the correctly classified examples.
Hence the more there are non-zero elements in a confusion matrix outside the diagonal, the more the classifier is prone to err.
Recall that in a learning process the objective is to learn a classifier $f\in\mathcal{F}$ with a low true error ({\it i.e.} with good generalization guarantees),  we are thus only interested in the errors of $f$.
Our objective is then to find $f$ leading to a confusion matrix with the more zero elements outside the diagonal.
Since the diagonal gives the conditional probabilities of 'correct' predictions, we propose to consider a different kind of confusion matrix by discarding the diagonal values.
Then the only non-zero elements of the new confusion matrix correspond to the examples that are misclassified by $f$.
For all $f\in\mathcal{F}$ we define the empirical and true confusion matrices of $f$ by respectively  $\mathbf{C}_S^f=(\hat{c}_{pq})_{1\leq p,q \leq Q}$ and $\mathbf{C}^f=(c_{pq})_{1\leq p,q \leq Q}$ such that: 
\begin{align}
\label{eq:MC_empirique}
\forall (p,q),\ \hat{c}_{pq}  &\egaldef  \left\{
\begin{array}{ll}
0&\textrm{if } q=p\\
   \displaystyle\hat{d}_{pq}& \textrm{otherwise},
\end{array}\right.\\
\label{eq:MC_reelle}
\forall (p,q),\ c_{pq}  &\egaldef  \left\{
\begin{array}{ll}
   0&\textrm{if } q=p\\
   d_{pq} = \PR_{(\xbf,y)\sim \mathfrak{D}}( f(\xbf)=q | p=y )&\textrm{otherwise}.
\end{array}\right.
\end{align} 

Note that if $f$ correctly classifies every example of a given sample $S$, then the empirical confusion matrix $\mathbf{C}_S^f$ is equal to $\mathbf{0}$.
Similarly, if $f$ is a perfect classifier over the distribution $\mathfrak{D}$, then the true confusion matrix is equal to $\mathbf{0}$.
Aiming at controlling the confusion matrix of a classifier is therefore a relevant task. More precisely, one may aim at a confusion matrix that
is `small', where `small' means as close to $\mathbf{0}$ as possible. As we shall see, the size of a confusion matrix will be measured by its operator norm.

\subsection{Main Result:  Confusion PAC-Bayes Bound for the Gibbs Classifier}

Our main result is a PAC-Bayes generalization bound that holds for
the Gibbs classifier $G_{\mathfrak{Q}}$ in the particular context of
multiclass prediction, where the empirical and true error measures are
respectively given by the confusion matrices defined by \eqref{eq:MC_empirique} and \eqref{eq:MC_reelle}.
In this case, we can define the true and the empirical confusion matrices of $G_{\mathfrak{Q}}$ respectively by:
\begin{align*} 
\mathbf{C}^{G_{\mathfrak{Q}}} =  \E_{f\sim\mathfrak{Q}} \E_{S\sim\mathfrak{D}_m} \mathbf{C}_S^f\qquad ; \qquad  \mathbf{C}_S^{G_{\mathfrak{Q}}} = \E_{f\sim\mathfrak{Q}}  \mathbf{C}_S^f.
\end{align*}

Given $f\sim \mathfrak{Q}$ and a sample $S\sim\mathfrak{D}_m$, our objective is to bound the difference between $\mathbf{C}^{G_{\mathfrak{Q}}}$ and $\mathbf{C}_S^{G_{\mathfrak{Q}}}$, the true and empirical errors of the Gibbs classifier. 
Remark the error rate $P(f(\xbf)\neq y)$ of a classifier $f$ might be
directly computed as the $1$-norm of ${}^t\Cbf^{f}\bf p$ with $\bf p$
the vector of prior class probabilities. A route to get results based
on the confusion matrix would then be to have a bound on the induced
$1$-norm of $\Cbf$ (which is defined by:
$\max\|{}^t\Cbf^f{\bf p}\|_1/\|{\bf p}\|_1$).
However, we do not have at hand concentration inequalities for the
$1$-norm of matrices and but we only have at our disposal such
concentration inequalities for the operator norm.
Since we have $\|{\bf u}\|_1\!\leq\!\sqrt{Q}\|{\bf u}\|_2$ for any $Q$-dimensional vector $\bf u$, we have that $P(f(\xbf)\neq y)\! \leq\! \sqrt{Q}\|\Cbf^f\|_{op}$, trying to minimize the operator norm of $\Cbf^f$ might be a relevant strategy to control the risk.
This norm will allow us to formally relate the true and empirical confusion matrices of the Gibbs classifier and also to provide a bound on  $\|\mathbf{C}^{G_{\mathfrak{Q}}}\|$ of the true confusion matrix.


Here is our main result.

\begin{theorem}
\label{theo:borne}
Let $X \subseteq \R^d$ be the input space, $Y = \{1,\dots,Q\}$ the output space, $\mathfrak{D}$ a distribution over $X \x Y$ (with $\mathfrak{D}_m$ the distribution of a $m$-sample) and $\mathcal{F}$ a family of classifiers from $X$ to $Y$.
Then for every prior distribution $\mathfrak{P}$ over $\mathcal{F}$ and any $\delta \in (0,1]$, we have: 
\begin{align*}
\PR_{S\sim \mathfrak{D}_m} &\Bigg\{\forall \mathfrak{Q}\textrm{ on }\mathcal{F}, \| \mathbf{C}_S^{G_{\mathfrak{Q}}} -  \mathbf{C}^{G_{\mathfrak{Q}}} \|  \leq \sqrt{\displaystyle\frac{8Q}{m_--8Q} \left[KL(\mathfrak{Q}||\mathfrak{P})+\ln\left(\frac{m_-}{4\delta}\right)\right]}\Bigg\}  \geq  1-\delta,
\end{align*}
where $m_-=\min_{y=1,\dots,Q}m_{y}$ corresponds to the minimal number of examples from $S$ which belong to the same class.
\end{theorem}

\begin{proof}Deferred to Section \ref{sec:preuve}.
\end{proof}
\noindent{}Note that, for all $y\in Y$, we need the following hypothesis: $m_{y}> 8Q$, which is not too strong a limitation.

Finally, we rewrite Theorem \ref{theo:borne} to have the size of the confusion matrix under consideration.
\begin{cor}
\label{cor:borne}
We consider the  hypothesis of the Theorem \ref{theo:borne}.
We have:
\begin{align*}
 \PR_{S\sim \mathfrak{D}_m} \Bigg\{ \forall \mathfrak{Q}\textrm{ on }\mathcal{F},\ \|  \mathbf{C}^{G_{\mathfrak{Q}}}\| \leq \|\mathbf{C}_S^{G_{\mathfrak{Q}}}\| +  \sqrt{\displaystyle\frac{8Q}{m_- - 8Q} \left[KL(\mathfrak{Q}||\mathfrak{P})+\ln\left(\frac{m_-}{4\delta}\right)\right]}\Bigg\}  \geq  1-\delta.
\end{align*}
\end{cor}
\begin{proof}
By application of the reverse triangle inequality 
$|\|\mathbf{A}\|-\|\mathbf{B}\||\leq\|\mathbf{A}-\mathbf{B}\|$ to Theorem \ref{theo:borne}.
\end{proof}

For a fixed prior $\Pfrak$ on $\Fcal$, both Theorem \ref{theo:borne} and Corollary \ref{cor:borne} yield a bound on the estimation (through the operator norm) of the true confusion matrix of the Gibbs classifier over all\footnote{This includes any $\Qfrak$ chosen by the learner after observing $S$.} posterior distribution $\mathfrak{Q}$ on $\Fcal$, though this is more explicit in the corollary.
Let the number of classes $Q$ be a constant, then the true risk is upper-bounded by the empirical risk of the Gibbs classifier and a term depending on the number of training examples, especially on the value $m_-$ which corresponds to the minimal quantity of examples that belong to the same class.
This means that the larger $m_-$, the closer the empirical confusion matrix of the Gibbs classifier to its true matrix. 
These bounds use first-order information and vary as $O(1/\sqrt{m_-})$, which is a typical rate of bounds not using second-order information.

\subsection{Upper Bound on the Risk of the Majority Vote Classifier}

Our multiclass upper bound given for the risk of Gibbs classifiers leads to an upper bound for the risk of Bayes classifiers in the following way by the Proposition \ref{prop:borneriskbayes}.
We recall that the Bayes classifier $B_{\Qfrak}$ is well known as majority vote classifier under a given posterior distribution $\Qfrak$. 
In the multiclass setting, $B_{\Qfrak}$ is such that for any example  it returns the majority class under the measure $\Qfrak$ and we define it as:
\begin{align}
\label{eq:bayes_multiclasse}
B_{\Qfrak}(\xbf) = \argmax_{c\in Y}\Big[\E_{f\in \posterior}\I(f(\xbf)=c)\Big].
\end{align}
We define the conditional Gibbs risk $R(G_{\posterior},p,q)$ and Bayes risk
$R(G_{\posterior},p,q)$ as
\begin{align}
R(G_{\posterior},p,q)& = \expectation_{\xbf\sim
  D_{|y=p}}\expectation_{f\sim\posterior}\I(f(\xbf)=q),\label{eq:x_gibbs}\\
R(B_{\posterior},p,q)&=\expectation_{\xbf\sim D_{|y=p}}\I\left(\argmax_{c\in Y} \Big[ \E_{f\in \posterior}\I(f(\xbf)=c) = q\Big]\right). \label{eq:x_bayes}
\end{align}
The former is the $(p,q)$ entry of $\Cbf^{G_{\posterior}}$ (if $p\neq
q$) and the latter is the $(p,q)$ entry of $\Cbf^{B_{\posterior}}$.

\begin{propo}
\label{prop:borneriskbayes}
Given $Q\geq 2$ the number of class.
The true conditional risk of the Bayes classifier  and the one of the Gibbs classifier   are related by the following inequality:
\begin{align}
\label{eq:borneriskbayes}
\forall (q,p),  R(B_{\Qfrak},p,q)  \leq QR(G_{\Qfrak},p,q).
\end{align}
\end{propo}
\begin{proof}
Deferred to Appendix.
\end{proof}
This proposition implies the following result.
\begin{cor}
\label{cor:bornebayes}
Given $Q\geq 2$ the number of class.
The true confusion matrix of the Bayes classifier $\Cbf^{B_{\Qfrak}}$  and the one of the Gibbs classifier $\Cbf^{G_{\Qfrak}}$ are related by the following inequality:
\begin{align}
\label{eq:bornebayes}
\| \Cbf^{B_{\Qfrak}} \| \leq Q\|\Cbf^{G_{\Qfrak}}\|.
\end{align}
\end{cor}
\begin{proof}
Deferred to Appendix.
\end{proof} 


\section{Proof of Theorem \ref{theo:borne}}
\label{sec:preuve}

This section gives the formal proof of Theorem \ref{theo:borne}.
We first introduce a concentration inequality for a sum of random square matrices.
This  allows us to deduce the PAC-Bayes generalization bound for
confusion matrices by following the same ``three step process''  as the one given in \cite{Mcallester03,Seeger02,Langford2005} for the classic PAC-Bayesian bound.

\subsection{Concentration Inequality for the Confusion Matrix}
\label{sec:ineg}
The main result of our work is based on the following corollary of a
result on the concentration inequality for a sum of self-adjoint
matrices given by \citet{Tropp2011} (see Theorem \ref{theo:tropp} in
Appendix) -- this theorem generalizes Hoeffding's inequality to the
case self-adjoint random matrices. 
The purpose of the following corollary is to restate the Theorem
\ref{theo:tropp} so that it carries over to matrices that are not
self-adjoint. It is central to us to have such a result as the
matrices we are dealing with, namely confusion matrices, are rarely symmetric.
\begin{cor} 
\label{cor:tropp}

Consider a finite sequence $\{\mathbf{M}_i\}$ of independent, random, 
square matrices of order $Q$, and let $\{a_i\}$ be a sequence of fixed scalars.
Assume that each random matrix satisfies $\E_i \mathbf{M}_i=\mathbf{0}$ and $\|\mathbf{M}_i\| \leq a_i$ almost surely.
Then, for all $\epsilon\geq 0$,
\begin{align}
\label{eq:ineg}
\PR\left\{ \|\sum_i \mathbf{M}_i\| \geq \epsilon \right\} \leq 2.Q.\exp\left(\frac{-\epsilon^2}{8\sigma^2}\right),
\end{align}
where $\sigma^2\egaldef \sum_i a_i^2$.
\end{cor}
\begin{proof}We want to verify the hypothesis given in  Theorem \ref{theo:tropp} in order to apply it.\\
Let $\{\mathbf{M}_i\}$ be a finite sequence of independent, random, square matrices of order $Q$ such that $\E_i \mathbf{M}_i=\mathbf{0}$ and let $\{a_i\}$ be a sequence of fixed scalars such that $\|\mathbf{M}_i\| \leq a_i$.
We consider the sequence $\{\mathcal{S}(\mathbf{M}_i)\}$ of random self-adjoint matrices with dimension $2Q$.
By the definition of the dilation, we directly obtain $\E_i \mathcal{S}(\mathbf{M}_i)=\mathbf{0}$.\\
From Equation \eqref{eq:preserve}, the dilation preserves the spectral information.
Thus, on the one hand, we have:
\begin{align*}
\|\sum_i \mathbf{M}_i\| = \lambda_{\max}\bigg(\mathcal{S}\Big( \sum_i \mathbf{M}_i\Big) \bigg) = \lambda_{\max}\Big(\sum_i \mathcal{S}(\mathbf{M}_i)\Big).
\end{align*}
On the other hand, we have: 
\begin{align*}
\|\mathbf{M}_i\|= \|\mathcal{S}(\mathbf{M}_i)\| = \lambda_{\max}\big(\mathcal{S}(\mathbf{M}_i)\big)  \leq a_i.
\end{align*}
To assure the hypothesis $\mathcal{S}(\mathbf{M}_i)^2\leqo \mathbf{A}_i^2$, we need to find a suitable sequence of fixed self-adjoint matrices $\{\mathbf{A}_i\}$ of dimension $2Q$ (where $\leqo$ refers to the semidefinite order on self-adjoint matrices).
Indeed, it suffices  to construct a diagonal matrix defined as $\lambda_{\max}\big(\mathcal{S}(\mathbf{M}_i)\big)\mathbf{Id}_{2Q}$ for ensuring $\mathcal{S}(\mathbf{M}_i)^2\leqo \big(\lambda_{\max}\big(\mathcal{S}(\mathbf{M}_i)\big)\mathbf{Id}_{2Q}\big)^2$.
More precisely, since for every $i$ we have  $\lambda_{\max}\big(\mathcal{S}(\mathbf{M}_i)\big) \leq a_i$,  we fix  $\mathbf{A}_i$ as a diagonal matrix with $a_i$ on the diagonal, {\it i.e.} $\mathbf{A}_i\egaldef a_i\mathbf{Id}_{2Q}$, with $\|\sum_{i}\mathbf{A}_i^2\|=\sum_{i}a_i^2=\sigma^2$.\\
Finally, we can invoke Theorem \ref{theo:tropp} to obtain the concentration inequality \eqref{eq:ineg}.
\end{proof}

In order to make use of this corollary,  we rewrite confusion matrices as sums of example-based confusion matrices.
That is, for each example $(\xbf_i,y_i)\in S$, we define its empirical confusion matrix by $\mathbf{C}^f_i=(\hat{c}_{pq}(i))_{1\leq p,q \leq Q}$ as follows:
\begin{align*}
\forall p,q, \hat{c}_{pq}(i) \egaldef \left\{
\begin{array}{ll}
0 &\textrm{ if } q=p\\
\displaystyle\frac{1}{m_{y_i}} \I( f(\xbf)=q)\I(y_i=p) &\textrm{otherwise}.
\end{array}\right.
\end{align*}
where $m_{y_i}$ is the number of examples of class $y_i\in Y$ belonging to $S$.
Given an example $(\xbf_i,y_i)\in S$, the example-based confusion matrix contains at most one non zero-element when $f$ misclassifies $(\xbf_i,y_i)$.
In the same way, when $f$ correctly classifies $(\xbf_i,y_i)$ then the example-based confusion matrix is equal to $\mathbf{0}$.
Concretely, for every sample $S=\{(\xbf_i,y_i)\}_{i=1}^m$ and every $f\in\mathcal{F}$, our error measure is then $$\mathbf{C}_S^f=\sum_{i=1}^m \mathbf{C}_i^f.$$
It naturally appears that we penalize only when $f$ errs.

Moreover, e further introduce the random square matrices $\mathbf{C'}_i^f=(\hat{c}_{pq}'(i))_{1\leq p,q \leq Q}$ defined by:
\begin{align}
\label{eq:chatprime}
\forall p,q, \hat{c}_{pq}'(i) \egaldef \left\{
\begin{array}{ll}
0 &\textrm{ if } \hat{c}_{pq}(i)=0\\
\displaystyle\frac{1}{m_{y_i}} \left(\I( f(\xbf_i)=q)\I(y_i=p) - \E_{S \sim (D)^m} \frac{1}{m_{y}}\I( f(\xbf)=q)\I(y=p)\right)  &\textrm{otherwise}.
\end{array}\right.
\end{align}
The term $\E_{S\sim (D)^m} \frac{1}{m_{y}}\I( f(\xbf)=q)\I(y=p)$, when $\hat{c}_{pq}(i)\ne 0$, is equivalent to the expectation (according to $S\sim (D)^m$) of the elements $\hat{c}_{pq}$ of $\mathbf{C}_S^f$ , such that $p=y_i$ and $q=h(\xbf_i)$.
Equation \eqref{eq:chatprime} is then equivalent to:
\begin{align*}
\forall p,q, \hat{c}_{pq}'(i) \egaldef \left\{
\begin{array}{ll}
0 &\textrm{ if } \hat{c}_{pq}(i)=0\\
 \hat{c}_{pq}(i) - \frac{1}{m_{y_i}} \E_{S \sim (D)^m} \hat{c}_{pq} &\textrm{otherwise}.
\end{array}\right.
\end{align*}
For sake of clarity, given an example $(\xbf_i,y_i)$ and for every $S\sim (D)^m$, we denote $\mathbf{C}_{S|i}^f$ the matrix with at most one non-zero element of coordinates $(p,q)$ equals to $\hat{c}_{pq}$ with $p=y_i$ and $q=h(\xbf_i)$.
Then, we obtain the following definition of $\mathbf{C'}_i^f$:
\begin{equation}
\label{eq:sarm}
\mathbf{C'}_i^f = \mathbf{C}^f_i - \frac{1}{m_{y_i}} \E_{S\sim\mathfrak{D}_m}\mathbf{C}_{S|i}^f,
\end{equation}
which verify $\E_i \mathbf{C'}_i^f= 0$. 

We have yet to find a suitable $a_i$ for a given $\mathbf{C'}_i^f$.
Let $\lambda_{{\max}_{i}}$ be the maximum singular value of $\mathbf{C'}_i^f$.
It is easy to verified that $\lambda_{{\max}_{i}} \leq \textstyle\frac{1}{m_{y_i}}$. 
Thus, for all $i$ we fix $a_i$ equal to  $\textstyle\frac{1}{m_{y_i}}$. 

Finally, with the introduced notations, Corollary \ref{cor:tropp} leads to the following concentration inequality:
\begin{align}
\label{eq:ourineg}
\PR
\left\{ \|\sum_{i=1}^m \mathbf{C'}_i^f\| \geq \epsilon \right\} \leq 2.Q.\exp\left(\frac{-\epsilon^2}{8\sigma^2}\right).
\end{align}

This inequality \eqref{eq:ourineg} allows us to demonstrate our Theorem \ref{theo:borne} by following the process of \cite{Mcallester03,Seeger02,Langford2005}.

\subsection{``Three Step Proof'' Of Our Bound}

First, thanks to concentration inequality \eqref{eq:ourineg}, we prove the following lemma.

\begin{lemma}
\label{lemme:bound}
Let $Q$ be the size of $\mathbf{C}_S^f$ and $\mathbf{C'}_i^f=\mathbf{C}^f_i - \frac{1}{m_{y_i}} \E_{S\sim\mathfrak{D}_m}\mathbf{C}_{S|i}^f$ defined as in \eqref{eq:sarm}. 
Then the following bound holds for any $\delta\in(0,1]$:
\begin{align*}
\nonumber \PR_{S\sim \mathfrak{D}_m}   \Bigg\{
\E_{f\sim \Pfrak}   \left[  \exp\left(\frac{1-8\sigma^2}{8\sigma^2} \|\sum_{i=1}^m \mathbf{C'}_i^f\|^2\right) \right]  \leq  \frac{2Q}{8\sigma^2\delta}  
\Bigg\} \geq 1-\delta
\end{align*}
\end{lemma}
\begin{proof}
For readability reasons, we note $\mathbf{C'}_S^f =  \sum_{i=1}^m \mathbf{C'}_i^f$.
If $Z$ is a real valued random variable  so that $\PR\left (Z\geq
  z\right)\leq k \exp(-n.g(z))$ with $g(z)$ non-negative, non-decreasing and $k$ a constant, then $\PR \left (\exp\left((n-1)g(Z) \right)\geq \nu\right) \leq \min(1, k \nu^{-n/(n-1)})$. 
We apply this to the concentration inequality \eqref{eq:ourineg}. 
Choosing $g(z)=z^2$ (non-negative), $z=\epsilon$, $n=\frac{1}{8\sigma^2}$ and $k=2Q$, we obtain the following result:
\begin{align*}
\PR
\left\{ \exp\left(\frac{1-8\sigma^2}{8\sigma^2}\|\mathbf{C'}_S^f\| \right)\geq \nu \right\}
 \leq \min\big(1, 2Q\nu^{-1/(1-8\sigma^2)}\big).
\end{align*}

Note that $\exp\Big(\frac{1-8\sigma^2}{8\sigma^2}\|\mathbf{C'}_S^f\| \Big)$ is always non-negative. Hence it allows us to compute its expectation as:
\begin{align*}
\E
\Bigg[\exp\Big(\frac{1-8\sigma^2}{8\sigma^2}\|\mathbf{C'}_S^f\| \Big)\Bigg] &= \int_0^\infty \PR
\left\{ \exp\Big(\frac{1-8\sigma^2}{(8\sigma^2)}\|\mathbf{C'}_S^f\| \Big)\geq \nu \right\}d\nu\\
 &\leq  2Q + \int_1^\infty 2Q\nu^{-1/(1-8\sigma^2)}d\nu\\
 &= 2Q - 2Q \frac{1- 8\sigma^2}{8\sigma^2}\Big[\nu^{-8\sigma^2/(1-8\sigma^2)}\Big]_1^\infty\\
&= 2Q + 2Q \frac{1- 8\sigma^2}{8\sigma^2}\\
&=  \frac{2Q}{8\sigma^2}.
\end{align*}

For a given classifier $f\in\mathcal{F}$, we have:
\begin{equation}
\label{eq:premarkov}
\E_{S\sim\mathfrak{D}^m}\left[\exp\left(\frac{1-8\sigma^2}{8\sigma^2}\|\mathbf{C'}_S^f\| \right)\right] \leq \frac{2Q}{8\sigma^2}.
\end{equation}

Then, if $\mathfrak{P}$ is a probability distribution over $\mathcal{F}$, Equation \eqref{eq:premarkov} implies that:
\begin{equation}
\label{eq:premarkov2}
\E_{S\sim\mathfrak{D}^m}\left[\E_{f\sim\mathfrak{P}}\exp\left(\frac{1-8\sigma^2}{8\sigma^2}\|\mathbf{C'}_S^f\| \right)\right] \leq \frac{2Q}{8\sigma^2}.
\end{equation}

Using Markov's inequality\footnote{see Theorem \ref{theo:markov} in Appendix.}, we obtain the result of the lemma.
\end{proof}

The second step to prove Theorem \ref{theo:borne} is to use the shift given in \cite{Mcallester03}.
We recall this result in the following lemma.

\begin{lemma}[Donsker-Varadhan inequality \citet{Donsker75}]
Given the Kullback-Leibler divergence\footnote{The KL-divergence is defined in Equation \eqref{eq:KL}.} $KL(\mathfrak{Q}\|\mathfrak{P})$ between two distributions $\mathfrak{P}$ and $\mathfrak{Q}$ and let $g(\cdot)$ be a function, we have:
\label{lemme:shift}
$$\E_{a\sim \Qfrak}\Big[g(b)\Big] \leq KL(\mathfrak{Q}\|\mathfrak{P}) + \ln\E_{x\sim \Pfrak}\Big[ \exp(g(b))\Big].$$
\end{lemma}
\begin{proof}See \cite{Mcallester03}.
\end{proof}
Recall that  $\mathbf{C'}_S^f=\sum_{i=1}^m\mathbf{C'}^f_i$.
With $g(b) = \frac{1-8\sigma^2}{8\sigma^2}b^2$ and $b=\|\mathbf{C'}_S^f\|$, Lemma \ref{lemme:shift} implies:
\begin{align}
\label{eq:shift}
 \E_{f\sim \mathfrak{Q}}\Bigg[\frac{1-8\sigma^2}{8\sigma^2}\|\mathbf{C'}_S^f\|^2\Bigg] 
 \leq   KL(\mathfrak{Q}\|\mathfrak{P})  +  \ln\E_{f\sim \mathfrak{P}} \Bigg[  \exp\left( \frac{1-8\sigma^2}{8\sigma^2}\|\mathbf{C'}_S^f\|^2 \right)  \Bigg].
\end{align}

The last step that completes the proof of Theorem \ref{theo:borne} consists in applying the result we obtained in Lemma \ref{lemme:bound} to  Equation \eqref{eq:shift}. 
Then, we have:
\begin{equation}
\label{eq:shift2}
\E_{f\sim \mathfrak{Q}}\left[\frac{1-8\sigma^2}{8\sigma^2}\|\mathbf{C'}_S^f\|^2\right] \leq KL(\mathfrak{Q}\|\mathfrak{P}) + \ln\frac{2Q}{8\sigma^2\delta}.
\end{equation}

Since  $g(\cdot)$ is clearly convex, we apply Jensen's
inequality\footnote{see Theorem \ref{theo:jensen} in Appendix.} to~\eqref{eq:shift2}. Then, with  probability at least $1-\delta$ over $S$, and for every distribution $\mathfrak{Q}$ on $\mathcal{F}$, we have:
\begin{equation}
\label{eq:shift3}
\Big(\E_{f\sim \mathfrak{Q}}\|\mathbf{C'}_S^f\|\Big)^2   \leq  \frac{8\sigma^2}{1 - 8\sigma^2}   \left( KL(\mathfrak{Q}\|\mathfrak{P})  +  \ln\frac{2Q}{8\sigma^2\delta} \right) .
\end{equation}
Since $\mathbf{C'}_S^f=\sum_{i=1}^m\left[\mathbf{C}^f_i - \E_{S\sim\mathfrak{D}_m}\mathbf{C}_i^f\right]$, then the bound \eqref{eq:shift3} is quite similar to the one given in Theorem \ref{theo:borne}.

We present in the next section, the calculations leading to our PAC-Bayesian generalization bound.

\subsection{Simplification}
We first compute the variance parameter:
 $$\sigma^2= \sum_{i=1}^ma_i^2.$$
For that purpose, in Section \ref{sec:ineg} we showed that for each $i\in\{1,\dots,m\}$,  we can choose $a_i=\frac{1}{m_{y_i}}$, where $y_i$ is the class of the $i$-th example and $m_{y_i}$ is the number of examples of class $y_i$.
Thus we have:
\begin{equation*}
\sigma^2 = \sum_{i=1}^m\frac{1}{m_{y_i}^2} = \sum_{y=1}^Q\sum_{i:y_i=y}\frac{1}{m_y^2}  = \sum_{y=1}^Q\frac{1}{m_y}.
\end{equation*}

For sake of simplification of Equation \eqref{eq:shift3} and since the term on the right side of this equation is an increasing function with respect to $\sigma^2$,  we propose to upper-bound  $\sigma^2$:
\begin{equation}
\label{eq:bornesigma}
 \sigma^2 = \sum_{y=1}^Q\frac{1}{m_y} \leq \frac{Q}{\min_{y=1,\dots,Q} m_y}.
\end{equation}
Let $m_{-} \egaldef \min_{y=1,\dots,Q} m_y$, then using Equation \eqref{eq:bornesigma}, we obtain the following bound from Equation \eqref{eq:shift3}:
\begin{align*}
\Big(\E_{f\sim \mathfrak{Q}}[\|\mathbf{C'}_S^f\|]\Big)^2 & \leq  \frac{8Q}{m_{-}-8Q} \Big(KL(\mathfrak{Q}\|\mathfrak{P}) + \ln\frac{m_{-}}{4\delta}\Big).
\end{align*}
Then:
\begin{align}
\label{eq:borneCprime}
\E_{f\sim \mathfrak{Q}}[\|\mathbf{C'}_S^f\|] & \leq  \sqrt{\frac{8Q}{m_{-}-8Q} \Big(KL(\mathfrak{Q}\|\mathfrak{P}) + \ln\frac{m_{-}}{4\delta}\Big)}.
\end{align}

It remains to replace $\mathbf{C'}_S^f =  \sum_{i=1}^m\left[\mathbf{C}^f_i - \E_{S\sim\mathfrak{D}_m}\mathbf{C}_i^f\right]$.
Recall that $\mathbf{C}^{G_{\mathfrak{Q}}}  =  \E_{f\sim\mathfrak{Q}} \E_{S\sim\mathfrak{D}_m} \mathbf{C}_S^f$ and $\mathbf{C}_S^{G_{\mathfrak{Q}}}  =  \E_{f\sim\mathfrak{Q}}  \mathbf{C}_S^f$, we obtain:

\begin{align}
\label{eq:Cprime}
\nonumber\E_{f\sim \mathfrak{Q}}[\|\mathbf{C'}_S^f\|] &
= \E_{f\sim \mathfrak{Q}}\left[\|\sum_{i=1}^m\left[\mathbf{C}_i^f - \frac{1}{m_{y_i}} \E_{S\sim\mathfrak{D}_m}\mathbf{C}_{S|i}^f\right]\|\right]\\
\nonumber&= \E_{f\sim \mathfrak{Q}}\left[\|\sum_{i=1}^m\left[\mathbf{C}_i^f\right] - \sum_{i=1}^m\left[
\frac{1}{m_{y_i}} \E_{S\sim\mathfrak{D}_m}\mathbf{C}_{S|i}^f\right]\|\right]\\
\nonumber&= \E_{f\sim \mathfrak{Q}}\left[\| \mathbf{C}_S^f - \E_{S\sim\mathfrak{D}_m} \mathbf{C}_S^f\|\right]\\
\nonumber &\geq \| \E_{f\sim \mathfrak{Q}}\left[ \mathbf{C}_S^f - \E_{S\sim\mathfrak{D}_m} \mathbf{C}_S^f\right]\|\\
\nonumber &=  \|  \E_{f\sim \mathfrak{Q}} \mathbf{C}_S^f - \E_{f\sim \mathfrak{Q}}\E_{S\sim\mathfrak{D}_m} \mathbf{C}_S^f\|\\
&= \|   \mathbf{C}_S^{G_{\mathfrak{Q}}} - \mathbf{C}^{G_{\mathfrak{Q}}}\|.
\end{align}

By substituting the left part of the inequality \eqref{eq:borneCprime} with the term \eqref{eq:Cprime}, we  find the bound of our Theorem \ref{theo:borne}.



\section{Discussion and Future Work}
\label{sec:discussion}
This work gives rise to many interesting questions, among which the following ones.


Some perspectives will be focused on instantiating our bound given in
Theorem \ref{theo:borne} for specific multi-class frameworks, such as
multi-class SVM \cite{Weston98,Crammer2002,Lee04} and multi-class
boosting (AdaBoost.MH/AdaBoost.MR \cite{Schapire2000}, SAMME
\cite{Zhu09}, AdaBoost.MM \cite{Mukherjee2010}).  Taking advantage of
our theorem while using the confusion matrices, may allow us to derive
new generalization bounds for these methods.

Additionally, we are interested in seeing how effective learning methods may be 
derived from the risk bound we propose.  For instance, in the
binary PAC-Bayes setting, the algorithm MinCq proposed by \citet{MinCq}
minimizes a bound depending on the first two moments of the margin of
the $\mathcal{Q}$-weighted majority vote.  From our Theorem
\ref{theo:borne} and with a similar study, we would like to design a
new multi-class learning algorithm and observe how sound such an algorithm 
could be. This would probably require the derivation of a Cantelli-Tchebycheff deviation inequality
in the matrix case.

Besides, it might be very interesting to see how the noncommutative/matrix concentration
inequalities provided by~\citet{Tropp2011} might be of some use for other kinds of learning
problem such as multi-label classification, label ranking problems or structured prediction issues.

Finally, the question of extending the present work to the analysis of algorithms learning (possibly infinite-dimensional) 
operators as \cite{abernethy09newapproach} is also very exciting.


\section{Conclusion}
\label{sec:conclusion}
In this paper, we propose a new PAC-Bayesian generalization bound that
applies in the multi-class classification setting.  The originality of
our contribution is that we consider the confusion matrix as an error
measure. Coupled with the use of the operator norm on matrices, we are
capable of providing generalization bound on the `size' of confusion
matrix (with the idea that the smaller the norm of the confusion
matrix of the learned classifier, the better it is for the
classification task at hand).  The derivation of our result takes
advantage of the concentration inequality proposed by
\citet{Tropp2011} for the sum of random self-adjoint matrices,  that
we directly adapt to square matrices which are not self-adjoint.

The main results are presented in Theorem \ref{theo:borne} and
Corollary \ref{cor:borne}.  The bound in Theorem \ref{theo:borne} is
given on the difference between the true risk of the Gibbs classifier
and its empirical error.  While the one given in Corollary
\ref{cor:borne} upper-bounds the risk of the Gibbs classifier by its
empirical error.
Moreover we have bound the risk of the Bayes classifier by the one of the Gibbs classifier.

An interesting point is that our bound depends on the minimal quantity $m_-$
of training examples belonging to the same class, for a given number
of classes.  If this value increases, {\it i.e.} if we have a lot of
training examples, then the empirical confusion matrix of the Gibbs
classifier tends to be close to its true confusion matrix. A point worth noting is that 
the bound varies as $O(1/\sqrt{m_-})$, which is a typical rate in bounds not using second-order information.

The present work gives rise to a few algorithmic and theoretical questions that we have discussed in the previous section.

\section*{Appendix}

\begin{theorem}[Concentration Inequality for Random Matrices \citet{Tropp2011}]
\label{theo:tropp}
Consider a finite sequence $\{\mathbf{M}_i\}$ of independent, random, self-adjoint matrices with dimension $Q$, and let $\{\mathbf{A}_i\}$ be a sequence of fixed self-adjoint matrices.
Assume that each random matrix satisfies $\E \mathbf{M}_i=\mathbf{0}$ and $\mathbf{M}_i^2\leqo \mathbf{A}_i^2$ almost surely.
Then, for all $\epsilon\geq 0$,
\begin{align*}
\PR\left\{ \lambda_{\max} \Big(\sum_i \mathbf{M}_i\Big) \geq \epsilon \right\} \leq Q.\exp\left(\frac{-\epsilon^2}{8\sigma^2}\right),
\end{align*}
where $\sigma^2\egaldef \|\sum_{i}\mathbf{A}_i^2\|$ and $\leqo$ refers to the semidefinite order on self-adjoint matrices.
\end{theorem}

\begin{theorem}[Markov's inequality]
\label{theo:markov}
Let $Z$ be a random variable and $z\geq 0$,  then:
$$
\PR{(|Z|\geq z)} \leq \frac{\E(|Z|)}{z}.
$$
\end{theorem}

\begin{theorem}[Jensen's inequality]
\label{theo:jensen}
Let $X$ be an integrable real-valued random variable and $g(\cdot)$ be a convex function, then:
$$
f(\E[Z]) \leq \E[g(Z)].
$$
\end{theorem}

\subsection*{Proofs of Proposition \ref{prop:borneriskbayes} and its Corollary \ref{cor:bornebayes}}
\label{sec:preuvebayes}

This section gives the formal proofs of  Proposition \ref{prop:borneriskbayes} and its Corollary \ref{cor:bornebayes}.

\subsubsection*{Proof of Proposition \ref{prop:borneriskbayes}}

Consider a labeled pair $(\xbf,y)$. Let us introduce the notation
$\gamma_q(\xbf)$ for $q\in Y$ such that:
$$\gamma_q(\xbf) =
\E_{f\sim\posterior}\I(f(\xbf)=q)=\sum_{f:f(\xbf)=q}\posterior(q).$$
Obviously,
$$\sum_{q\in Y}\gamma_q(\xbf)=1.$$
Recall that 
the conditional Gibbs risk $R(G_{\posterior},p,q)$ and Bayes risk
$R(G_{\posterior},p,q)$ are defined as:
\begin{align}
R(G_{\posterior},p,q)& = \expectation_{\xbf\sim
  D_{|y=p}}\expectation_{f\sim\posterior}\I(f(\xbf)=q)=\expectation_{\xbf\sim
  D_{|y=p}}\gamma_q(\xbf),\tag{\ref{eq:x_gibbs}}\\
R(B_{\posterior},p,q)&=\expectation_{\xbf\sim D_{|y=p}}\I(\argmax_{c\in
Y}\gamma_c(\xbf)=q)\tag{\ref{eq:x_bayes}}
\end{align}
The former is the $(p,q)$ entry of $\Cbf^{G_{\posterior}}$ (if $p\neq
q$) and the latter is the $(p,q)$ entry of $\Cbf^{B_{\posterior}}$.

For $q\neq y$ to be predicted by the majority vote classifier, it is
necessary and sufficient that
$$\gamma_q(\xbf)\geq \gamma_c(\xbf),\;\forall c\in Y,c\neq q.$$
This might be equivalently rewritten as:
\begin{equation}
\I(\argmax_c\gamma_c(\xbf)=q)=\I(\operatorname{\wedge}_{c,c\neq
  q}\gamma_q(\xbf)\geq\gamma_c(\xbf))
\label{eq:bayes_errs}
\end{equation}
(note that the expectation of the left-hand side which respect to $D_{|y=p}$ is
$R(B_{\posterior},p,q)$ ---cf.~\eqref{eq:x_bayes}). Now remark that:
\begin{align*}
\I(\operatorname{\wedge}_{c,c\neq
  q}\gamma_q(\xbf)\geq\gamma_c(\xbf))=1&\Leftrightarrow
\gamma_q(\xbf)-\gamma_c(\xbf),\;\forall c\in Y,c\neq q\\
&\Rightarrow\sum_{c\in Y,c\neq
  q}(\gamma_q(\xbf)-\gamma_c(\xbf))\geq 0\\
&\Leftrightarrow\sum_{c\in Y,c\neq q}\gamma_q(\xbf)-\sum_{c\in Y,c\neq
  q}\gamma_c(\xbf)\geq 0\\
&\Leftrightarrow (Q-1)\gamma_q(\xbf) - (1-\gamma_q(\xbf))\geq 0\\
&\Leftrightarrow \gamma_q(\xbf)\geq \frac{1}{Q}.
\end{align*}
where we have used $\sum_{c\in Y}\gamma_c(\xbf)=1$ in the next to last
line. This means that:
$$\I(\operatorname{\wedge}_{c,c\neq
  q}\gamma_q(\xbf)\geq\gamma_c(\xbf))=1\Rightarrow \indicator{\gamma_q(\xbf)\geq \frac{1}{Q}}=1,$$
from which we get: 
$$\I(\operatorname{\wedge}_{c,c\neq
  q}\gamma_q(\xbf)\geq\gamma_c(\xbf))\leq \indicator{\gamma_q(\xbf)\geq
\frac{1}{Q}},$$
that is, by virtue of~\eqref{eq:bayes_errs}:
$$\indicator{\argmax_c\gamma_c(\xbf)=q} \leq \indicator{\gamma_q(\xbf)\geq
\frac{1}{Q}}.$$
We then may use that $\gamma\geq \theta\indicator{\gamma\geq 1/Q},\forall\gamma\in[0,1],\theta\in[0,1]$, as
illustrated on Figure~\ref{fig:step}, to obtain
$$\frac{1}{Q}\indicator{\gamma_q(\xbf)\geq
\frac{1}{Q}}\leq \gamma_q(\xbf)\Leftrightarrow \indicator{\gamma_q(\xbf)\geq
\frac{1}{Q}}\leq Q\gamma_q(\xbf),$$
and, combining with the previous inequality:
$$\indicator{\argmax_c\gamma_c(\xbf)=q}\leq Q\gamma_q(\xbf).$$
\begin{figure}[tb]
\begin{center}
\includegraphics[width=0.5\textwidth]{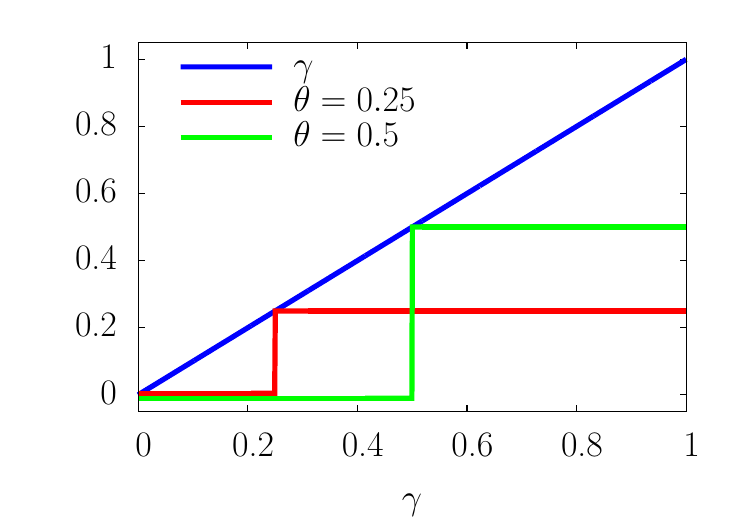}
\end{center}
\caption{Plot of $\gamma\mapsto
  \theta\indicator{\gamma\geq
    \theta},$ for $\theta=0.25$ (red) and $\theta=0.5$ (green). 
  Observe that $\gamma\geq \theta\indicator{\gamma\geq
    \theta}$, $\forall\theta\in[0,1]$.\label{fig:step}}
\end{figure}
Taking the expectation of both sides with respect to $\xbf\sim
D_{|y=p}$, we get:
$$R(B_{\posterior},p,q)\leq Q R(G_{\posterior},p,q).$$

\subsubsection*{Proof of Corollary \ref{cor:bornebayes}}

From the definitions of $R(G_{\posterior},p,q)$ \eqref{eq:x_gibbs} and  $R(B_{\posterior},p,q)$ \eqref{eq:x_bayes}, we directly obtain from Proposition \ref{prop:borneriskbayes}:
\begin{align}
\label{eq:cor1}
\Cbf^{B_{\Qfrak}} \leq Q\Cbf^{G_{\Qfrak}}.
\end{align}
We dilate  $\Cbf^{B_{\Qfrak}}$ and $Q\Cbf^{G_{\Qfrak}}$, then \eqref{eq:cor1} is rewritten as:
\begin{align*}
\Scal(\Cbf^{B_{\Qfrak}}) \leq \Scal(Q\Cbf^{G_{\Qfrak}}).
\end{align*}
Since all component of a confusion matrix are positive, we have $0\leq \Scal(\Cbf^{B_{\Qfrak}}) \leq \Scal(Q\Cbf^{G_{\Qfrak}})$.
We can thus apply the property \eqref{eq:sp}. We obtain:
\begin{align}
\label{eq:cor3}
\lambda_{\max}(\Scal(\Cbf^{B_{\Qfrak}})) \leq \lambda_{\max}(\Scal(Q\Cbf^{G_{\Qfrak}})).
\end{align}
Then, with property \eqref{eq:preserve}, \eqref{eq:cor3} is rewritten as:
\begin{align*}
\|\Cbf^{B_{\Qfrak}}\| \leq \|Q\Cbf^{G_{\Qfrak}}\|.
\end{align*}
Finally, by application of \eqref{eq:scalar}:
\begin{align*}
\|\Cbf^{B_{\Qfrak}}\| \leq Q\|\Cbf^{G_{\Qfrak}}\|.
\end{align*}

\bibliography{biblio}
\bibliographystyle{apalike}

\end{document}